\documentclass[letterpaper, 10 pt, conference]{ieeeconf}  
\IEEEoverridecommandlockouts                              
\usepackage{graphics} 
\usepackage{epsfig} 
\usepackage{mathptmx} 
\usepackage{times} 
\usepackage{amsmath} 
\usepackage{amssymb}  
\usepackage{psfrag} 
\usepackage{subcaption,tikz} 
\usepackage[english]{babel}
\usepackage{mathtools,url}
\usepackage{algorithm,algorithmic}
\newtheorem{theorem}{Theorem}[section]

\newtheorem{lemma}[theorem]{Lemma}

\newcommand{\CASE}[1]{\STATE \textbf{case} #1\textbf{:} \begin{ALC@g}}
\newcommand{\ENDCASE}{\end{ALC@g}}

\newcommand{\DEFAULT}{\STATE \textbf{default:} \begin{ALC@g}}
\newcommand{\ENDDEFAULT}{\end{ALC@g}}
\newcommand{\DEFAULTLINE}[1]{\STATE \textbf{default:} }

\title{\LARGE \bf
Safety Barrier Certificates for Heterogeneous Multi-Robot Systems*
}
\author{Li Wang, Aaron Ames, and Magnus Egerstedt$^\dagger$
\thanks{*The work by the first and third authors was sponsored by Grant No. N0014-15-1-2115 from the U.S. Office for Naval Research, and the work of the second author is supported by NSF CPS-1239055.}
\thanks{$^\dagger$Li Wang, Aaron Ames and Magnus Egerstedt are with the school of Electrical and Computer Engineering, Georgia Institute of Technology, Atlanta, GA 30332, USA.
        Email: {\tt\small \{liwang, ames, magnus\}@gatech.edu}}%
}

\begin{document}
\maketitle
\thispagestyle{empty}
\pagestyle{empty}

\begin{abstract}
This paper presents a formal framework for collision avoidance in multi-robot systems, wherein an existing controller is modified in a minimally invasive fashion to ensure safety. We build this framework through the use of control barrier functions (CBFs) which guarantee forward invariance of a safe set; these yield safety barrier
certificates in the context of heterogeneous robot dynamics subject to acceleration bounds. Moreover, safety barrier certificates are extended to a distributed control framework, wherein neighboring agent dynamics are unknown, through local parameter identification. The end result is an optimization-based controller that formally guarantees collision free behavior in heterogeneous multi-agent systems by minimally modifying the desired controller via safety barrier constraints. This formal result is verified in simulation on a multi-robot system consisting of both ``cumbersome" and ``agile" robots, is demonstrated experimentally on a system with a Magellan Pro robot and three Khepera III robots.
\end{abstract}

\section{INTRODUCTION}
When designing coordinated controllers for teams of mobile robots, the primary control objective tends to drive the behavior of the team so as to realize tasks such as achieving and maintaining formations, covering areas, or collective transport \cite{bullo2009distributed, mesbahi2010graph}. {\it Safety}, in terms of collision-avoidance, is then added as a secondary controller that overrides the existing controllers on individual robots if they are about to collide, e.g., following the behavior-based control paradigm \cite{arkin1998behavior}. As a result, what is actually deployed is not always what the design calls for, and as the robot density increases, the team spends more and more time avoiding collisions as opposed to progressing toward the primary design objective. 

One remedy to this problem is to make collision-avoidance an explicit part of the design. This, however, means that the already established, coordinated multi-robot controllers in the literature \cite{bullo2009distributed, mesbahi2010graph, ren2008distributed} are no longer valid and must be revisited. An alternative view, as is for example pursued in \cite{tomlin1998conflict} for two aircrafts performing optimal evasive maneuver, is to introduce a minimally invasive collision-avoidance controller, i.e., a controller that only changes the original control program when it is absolutely necessary.  But the heavy computation associated with solving the Hamilton-Jacobian-Bellman Equations prohibits the applicability of \cite{tomlin1998conflict} to large-scale mutli-robot systems.  Similarly, the concept of ``velocity obstacle" was developed in \cite{van2011reciprocal} to generate collision free trajectory in cluttered multi-agent workspace, while the constant velocity assumption severely limits available control options. This approach was pursued in \cite{borrmann2015Swarm}, where the main idea is to let the actual control input associated with Robot $i$, $u_i$, be as close to the designed control input $\hat{u}_i$ in a least-squares sense, subject to safety contstraints. 

The way that safety constraints were encoded in \cite{borrmann2015Swarm} was using distributed {\it barrier functions}, that prevented the robots from entering unsafe states. This line of inquiry is continued in this paper, but in the context of {\it heterogeneous} robot teams. In particular, the barrier functions in \cite{borrmann2015Swarm} were symmetric in the sense that the responsibility for avoiding collisions was shared in an equitable manner among the robots. But, in a heterogeneous multi-robot system, not all agents are equally nimble and can respond to potential collisions in the same way, due to such factors as different maximal accelerations. In this paper, we pursue this question and we show how barrier functions can be used also for teams of heterogeneous networks, even when the robots are unaware of which class neighboring robots belong.

The reason why heterogeneous multi-agent systems are of importance is that they, through the robots' diverse set of capabilities, can solve some tasks more effectively than their homogeneous counter parts, i.e., \cite{abbas2011distribution}. Moreover, heterogeneity already exists in many systems, such as transportation systems with automobiles and trucks \cite{arasan2005methodology}, multirobot systems with ground and aerial robots \cite{ding2010multi}, mobile sensor network with nodes with varying 
locomotion capabilities \cite{parker2003heterogeneous}, just to name a few. As such, collision-avoidance algorithms must be extended also to heterogeneous systems. Yet, such an extension is not straightforward in that agents with ``aggressive", ``neutral" or even ``timid" behaviors must be able to respond to possible collisions in dramatically different manners.

Motivated by these considerations, this paper extends previous work on safety barrier certificates in \cite{borrmann2015Swarm} in two important directions. First, we propose a provably safe way to decentralize the barrier certificates that explicitly takes the agents' heterogeneous dynamics into account. In this paper, the robotic swarm is heterogeneous in the sense that agents have different acceleration limits (agile or cumbersome), and use different barrier certificate parameters (aggressive, neutral or conservative). Second, we formally ensure safety of the robotic swarm when no prior information about neighboring agents' dynamical properties is provided. To achieve this, the agents will have to estimate the dynamical properties of neighboring agents with local measurements, and update online their barrier certificate parameters to generate more reasonable evasive maneuver. The enabling technique for this heterogeneous safety barrier certificates is Control Barrier Function \cite{ames2014CBF, Xu2015Robustness}. Control Barrier Function is similar to Control Lyapunov Function in that it provides a way to guarantee the forward invariance of the safety set without computing the system's reachable set. A Quadratic Program (QP) based controller with safety barrier constraints is developed to check the safety of pre-designed control strategy, and generate minimally-invasive and collision-free trajectory.

The remainder of this paper is organized as follows. Section \ref{sec:cbf} revisits the concepts of (zeroing) control barrier function, which is incorporated into the optimization based controller as safety barrier constraints. Heterogeneous safety barrier certificates are then constructed in Section \ref{sec:heterosfc} to generate collision free behaviors for agents' with different dynamical capabilities. Incorporating unknown parameters into heterogeneous barrier certificates without losing safety guarantee is the topic of Section \ref{sec:estsfc}. Simulation and experimental results for heterogeneous barrier certificates are presented in Section \ref{sec:sim} and \ref{sec:exp}. At last, we conclude the paper with a summary and discussion of future work in Section \ref{sec:conclude}.

\section{Background: Control Barrier Functions} \label{sec:cbf}
In this section, we will review the fundamentals of Control Barrier Functions (CBFs), which is employed as a means to ensure that the robots execute collision-free trajectories. CBFs are conceptually similar to Control Lyapunov Functions (CLFs) in that they can be used to guarantee desired system properties without explicitly having to compute the forward reachable set. Analogously to CLFs, by constraining the time derivative of the CBFs within prescribed bounds, CBFs can formally guarantee the forward invariance of a desired set, e.g. safe set. 

The fundamental idea behind CBFs is thus to design them in such a way that the agents always remain in the safe set. We are particularly interested in control affine dynamic systems because they result in affine safety barrier constraints, which can be incorporated into simple QP based controllers. Even though the main focus of this paper is on double integrator dynamics, we start the exposition with general control affine case. In particular, consider a nonlinear control system in affine form
\begin{equation}
\dot{x}=f(x)+g(x) u
\label{eqn:system}
\end{equation}
where $x \in \mathbb{R}^n$ and $u \in U \subset \mathbb{R}^m$, with $f$ and $g$ locally Lipschitz continuous. Note that, for the sake of simplicity, we will assume that (\ref{eqn:system}) is forward complete, i.e. solutions $x(t)$ are defined for all $t\geq 0$.

Suppose now that we have a set $\mathcal{C}\subset\mathbb{R}^n$ where we wish the state of the robot to remain. The goal is thus to design a controller $u$ that guarantees the forward invariance of set $\mathcal{C}$, i.e., solutions to (\ref{eqn:system}) that start in set $\mathcal{C}$, stay in set $\mathcal{C}$ for all time. We will assume that the set $\mathcal{C}$, boundary $\partial\mathcal{C}$ and interior $\text{Int}(\mathcal{C})$ can be defined as levels sets to a particular function $h(x)$,
\begin{equation}
\begin{aligned}
\mathcal{C} &= \{x\in\mathbb{R}^n\,:h(x)\geq 0\},\\
\partial\mathcal{C} &= \{x\in\mathbb{R}^n\,:h(x)=0\},\\
\text{Int}(\mathcal{C}) &= \{x\in\mathbb{R}^n\,:h(x)>0\},\\
\end{aligned}
\label{eqn:admSet}
\end{equation}
and we have the following definition that allows us to be precise about what safety entails, as was done in \cite{Xu2015Robustness},
\textit{Definition 1}: Given a dynamical system (\ref{eqn:system}) and the set $\mathcal{C}$ defined by (\ref{eqn:admSet}) for a continuously differentiable function $h  : \mathbb{R}^n\to\mathbb{R}$, if there exist a locally Lipschitz extended class $\mathcal{K}$ function $\alpha$ (strictly increasing and $\alpha(0) =0$) and a set $\mathcal{C}\subseteq\mathcal{D}\subset\mathbb{R}^n$ such that, for all $x\in\mathcal{D}$,
\begin{align}
\underset{u\in U}{\text{inf}}\left[ L_fh(x)+L_gh(x)u+\alpha(h(x))\right] \geq 0
\label{eqn:reqB1}
\end{align}
then the function $h$ is called a Zeroing Control Barrier Function (ZCBF) defined on set $\mathcal{D}$.

\medskip
Now, given a ZCBF $h$, the set of feasible control inputs is
\begin{align}
\label{eqn:reqB2}
K (x) =
     \left \lbrace u\in U\ : \ L_fh(x)+L_gh(x)u+\alpha(h(x))\geq 0\right\rbrace \nonumber
\end{align}
and in \cite{Xu2015Robustness}, the following key result was obtained;

\textit{Theorem \cite{Xu2015Robustness}}: \textit{Given a set $\mathcal{C} \subset \mathbb{R}^n$ defined by (\ref{eqn:admSet}) and a ZCBF $h$ defined on $\mathcal{D}$ with $\mathcal{C}\subseteq\mathcal{D}\subset\mathbb{R}^n$, any Lipschitz continuous controller $u\colon\mathcal{D}\to\mathbb{R}$ such that $u\in K(x)$ for the system (\ref{eqn:system}) renders the set $\mathcal{C}$ forward invariant.}

ZCBFs also imply asymptotic stability of set $\mathcal{C}$, which provides favorable robustness properties with respect to different perturbations on system (\ref{eqn:system}) \cite{Xu2015Robustness}. If the state $x$ is perturbed into $\mathcal{D}\setminus\mathcal{C}$, it will converge asymptotically back into set $\mathcal{C}$.

In this paper, we will choose $\alpha(h(x)) = \gamma h^3(x)$ for defining our ZCBF candidate, which means the designed controller $u$ needs to satisfy the following constraint.
\begin{equation}
\ L_fh(x)+L_gh(x)u+\gamma h^3(x)\geq 0 \label{eq:zcbf}
\end{equation}

\section{Heterogeneous Safety Barrier Certificates} \label{sec:heterosfc}
This section focuses on constructing the decentralized heterogeneous safety barrier certificates that take into account the heterogeneity in agents' dynamical properties. Importantly, in an effort to reduce the amount of information required when executing barrier certificates, we will explore safety guarantees subject to unknown parameters of neigboring agents in Section \ref{sec:estsfc}.

\subsection{Problem Formulation} \label{set:pform}
Consider a heterogeneous robotic swarm containing $N$ mobile agents represented with set $\mathcal{M}=\lbrace 1,2,\dots,N\rbrace$, the robot agent $i\in\mathcal{M}$ is modelled with double integrator dynamics (\ref{eqn:dint}). Agents in the robotic swarm are heterogeneous in the sense that each of them has different dynamical capability, which is modelled with different speed and acceleration limit:
\begin{equation}\label{eqn:dint}
     \begin{bmatrix}
       \dot{\mathbf{p}}_i  \\[0.3em]
       \dot{\mathbf{v}}_i  \\[0.3em]
     \end{bmatrix}
     = \begin{bmatrix}
       0 & I_{2\times2}  \\[0.3em]
       0 & 0
     \end{bmatrix}  \begin{bmatrix}  \mathbf{p}_i  \\[0.3em]
       \mathbf{v}_i  \end{bmatrix}
       +  \begin{bmatrix}  0 \\[0.3em]    I_{2\times2} \end{bmatrix}  \mathbf{u}_i,
\end{equation}
where $\mathbf{p}_i\in \mathbb{R}^2$, $\mathbf{v}_i\in \mathbb{R}^2$, and $\mathbf{u}_i\in \mathbb{R}^2$ are the position, velocity, and acceleration of agent $i$ respectively. The ensemble form is $\mathbf{p}\in \mathbb{R}^{2N}$, $\mathbf{v}\in \mathbb{R}^{2N}$, and $\mathbf{u}\in \mathbb{R}^{2N}$. The speed and acceleration limits of agent $i$ are $\|\mathbf{v}_i\|_p\leq \beta_i$ and $\|\mathbf{u}_i\|_p\leq \alpha_i$, where $\|\cdot\|_p$ is vector $p-$norm depending on actual robot model. The relative position and relative velocity between agent $i$ and $j$ are denoted as $\Delta \mathbf{p}_{ij}= \mathbf{p}_i - \mathbf{p}_j$ and $\Delta \mathbf{v}_{ij} = \mathbf{v}_i - \mathbf{v}_j$. 

Next, we need to formulate an appropriate safe operation set $\mathcal{C}$ that characterizes safety of the robotic swarm. A pairwise safety constraint (\ref{eqn:2ndsafe}) is adopted to ensure that agents will always keep safety distance $D_s$ away from each other in dangerous scenarios while the maximum braking force is being applied; this results in the constraint:
\begin{equation}\label{eqn:2ndsafe}
- \frac{\Delta \mathbf{p}_{ij}^T}{\|\Delta \mathbf{p}_{ij}\|}\Delta \mathbf{v}_{ij} \leq \sqrt{2(\alpha_i + \alpha_j)(\|\Delta \mathbf{p}_{ij}\|-D_s)} , \forall i \neq j
\end{equation}

As illustrated in Fig \ref{fig:dpdv}, the normal component ($\Delta \bar{v} = \frac{\Delta \mathbf{p}_{ij}^T}{\|\Delta \mathbf{p}_{ij}\|}\Delta \mathbf{v}_{ij}$) of the relative velocity between agent $i$ and $j$ is the component that might lead to collision. For instance, it is considered safe when two neighboring agents' relative velocity is perpendicular to their relative position ($\Delta \bar{v}=\mathbf 0$).
\vspace{-0.1in}
\begin{figure}[h]
  \centering
  \resizebox{2in}{!}{\includegraphics{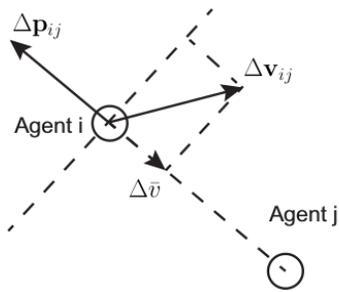}}
  \caption{Relative position and velocity between agent $i,j$}
  \label{fig:dpdv}
\end{figure}

The safety constraint (\ref{eqn:2ndsafe}) is derived by regulating the normal component of the relative velocity $\Delta \bar{v}$, while the tangent component is unregulated. When two agents are actively avoiding collision, the maximum relative braking acceleration is $(\alpha_i + \alpha_j)$. The safety requirement is to maintain safety distance $D_s$ away from each other when the maximum relative braking acceleration $(\alpha_i + \alpha_j)$ is applied in dangerous scenarios, which leads to:
\vspace{-0.05in}
\begin{eqnarray}
\|\Delta \mathbf{p}_{ij}\| - \frac{(\Delta \bar{v})^2}{2(\alpha_i + \alpha_j)} \geq  D_s, \quad \forall i \neq j. \label{eqn:consderiv}
\end{eqnarray}
Note that when two agents are moving closer to each other ($\Delta \bar{v} <0$), (\ref{eqn:consderiv}) regulates how fast the approaching speed could be; when they are moving away from each other $\Delta \bar{v} \geq 0$), no constraint is enforced because safety is not endangered. Those two cases combined together gives the safety constraint in (\ref{eqn:2ndsafe}). Therefore, we can formally define the safe operation set $\mathcal{C}$.
\vspace{-0.1in}
\begin{eqnarray}
\mathcal{C}_{ij} &=& \{(\mathbf{p}_{i},\mathbf{v}_{i}) |h_{ij} = \sqrt{2(\alpha_i + \alpha_j)(\|\Delta \mathbf{p}_{ij}\|-D_s)} \nonumber \\ 
&& +  \frac{\Delta \mathbf{p}_{ij}^T}{\|\Delta \mathbf{p}_{ij}\|}\Delta \mathbf{v}_{ij} \geq 0  \} , j\neq i \label{eqn:cij} \\
\mathcal{C} &=& \underset{{i \in \mathcal{M}}}{\prod} \left\{{\bigcap_{\substack{j \in \mathcal{M}\\ j\neq i}}\mathcal{C}_{ij}}\right\}  \label{eqn:c}
\end{eqnarray}
where $h_{ij}$, short for $h_{ij}(\Delta \mathbf{p}_{ij},\Delta \mathbf{v}_{ij})$, is a ZCBF candidate to ensure that the pairwise constraint (\ref{eqn:2ndsafe}) always holds. $\underset{i \in \mathcal{M}}{\prod}$ is the Cartesian product over the states of all agents in the set of robots.

\noindent \textit{Definition 2}: The robotic swarm represented with set $\mathcal{M}$ with dynamics given in (\ref{eqn:dint}) is defined to be \textit{safe} if the state $(\mathbf p, \mathbf v)$ of the system stays in set $\mathcal{C}$ for all time. 

According to \textit{Definition 2}, the robotic swarm needs to simultaneously satisfy all the pairwise safety constraints to ensure safety. ZCBF constraints are constructed to guarantee the forward invariance of the safe operation set $\mathcal{C}$, i.e. there are the following pairwise CBF constraints:
\begin{equation}\label{eqn:zcbfc}
L_fh_{ij} + L_gh_{ij}\mathbf u+\gamma h_{ij}^3\geq 0, \forall i \neq j,
\end{equation}

\begin{theorem}
The robotic swarm represented with set $\mathcal{M}$ is \textit{safe}, if the control variable $\mathbf{u}$ satisfies all the pairwise ZCBF constraints in (\ref{eqn:zcbfc}).
\end{theorem}
\begin{proof}
If the control variable $\mathbf u$ satisfies the pairwise ZCBF constraints in (\ref{eqn:zcbfc}), then $h_{ij}$ is a valid ZCBF with $\alpha(x) = \gamma x^3$ according to \textit{Definition 1}. Following \textit{Theorem \cite{Xu2015Robustness}}, the forward invariance of set $\mathcal{C}$ is guaranteed, which means the robotic swarm denoted with set $\mathcal{M}$ is \textit{safe}.
\end{proof}
Combining (\ref{eqn:cij}) with (\ref{eqn:zcbfc}), the ZCBF constraint can be rewritten as;
\vspace{-0.1in}
\begin{multline}\label{eqn:cbf3}
- \Delta \mathbf{p}_{ij}^T\Delta \mathbf{u}_{ij} \leq  \gamma h_{ij}^3\|\Delta \mathbf{p}_{ij}\| - \frac{(\Delta \mathbf{v}_{ij}^T\Delta \mathbf{p}_{ij})^2}{\|\Delta \mathbf{p}_{ij}\|^2}  \\+ \|\Delta \mathbf{v}_{ij}\|^2 + \frac{(\alpha_i + \alpha_j)\Delta \mathbf{v}_{ij}^T\Delta \mathbf{p}_{ij} }{\sqrt{2(\alpha_i + \alpha_j)(\|\Delta \mathbf{p}_{ij}\|-D_s)}}  , \quad \forall i\neq j
\end{multline}
This safety barrier constraint can be represented as linear constraints on the control variable $\mathbf{u}$ as $A_{ij}\mathbf u\leq b_{ij}$, where
$$
A_{ij}=[0, ..., \underbrace{-\Delta \mathbf{p}_{ij}^T}_{\text{agent} ~i}, ..., \underbrace{\Delta \mathbf{p}_{ij}^T}_{\text{agent} ~j},\\ ..., 0 ],
$$
and $b_{ij}$ is the right side of (\ref{eqn:cbf3}).

The safety barrier constraints assembled together, termed the {\it safety barrier certificates}, defines the space of permissible controls. The objective of the safety barrier certificates is to validate the safety of pre-designed control strategy $\hat{\mathbf u}$, and interfere with minimal impact to the desired strategy when collision is truly imminent. The goals of collision avoidance and minimal interference are combined together using QP:
\begin{equation}
\label{eqn:QPcontroller}
 \begin{aligned}
\mathbf{u}^* =  & \:\: \underset{\mathbf u}{\text{argmin}}
 & & J(\mathbf u) = \sum_{i=1}^{N} \|{\mathbf u}_{i} - \hat{\mathbf u}_{i} \|^2 &\\
 & \qquad \text{s.t.}
 & & A_{ij}\mathbf{u} \leq b_{ij},\: \qquad  \forall i\neq j,  &  \\
 &
 & &    \| \mathbf u_i\|_\infty  \leq \alpha_i,\:\qquad \forall i \in\mathcal{M}
 \end{aligned}
\end{equation}
where the acceleration limit of agent $i$ is defined with $\infty$-norm for simplicity. This QP based controller follows pre-designed control strategy $\hat{\mathbf u}$ when the system is safe; takes over and computes the closest permissible control in a least-squares sense when collision is about to happen. Note that this QP based controller is a centralized controller, demanding centralized computation, which provides a starting point for decentralized heterogeneous barrier certificates.

\subsection{Decentralized Heterogeneous Barrier Certificates} \label{sec:DecHeteroBar}
Centralized safety barrier certificates face significantly increased communication and computation burden when the size of robotic swarm grows. It is desirable to have decentralized barrier certificates that act only based on local information, while safety is still guaranteed. In the heterogeneous robotic swarm, agents with different acceleration limits have different capabilities to avoid collision. Thus we propose three different strategies to decentralize the safety barrier certificates to each agent based on their acceleration limits. Motivated by the fact that agents with higher acceleration limits are more agile, these agents are assigned with larger portion of the admissible control space. 
\begin{enumerate}
  \item Strategy A partitions the increase rate of ZCBF $\dot{h}_{ij}=\frac{\partial h_{ij}}{\partial \mathbf{x}_i}^T\dot{\mathbf{x}}_i+\frac{\partial h_{ij}}{\partial \mathbf{x}_j}^T\dot{\mathbf{x}}_j$ to two robot agents $i$ and $j$, where $\mathbf x_i = \left[ \begin{array}{c} \mathbf p_i\\ \mathbf v_i \end{array} \right]$ is the state of agent $i$.  
 \begin{eqnarray} \label{eqn:cbfdeca}
-\frac{\partial h_{ij}}{\partial \mathbf{x}_i}^T\dot{\mathbf{x}}_i \leq \frac{\alpha_i}{\alpha_i+\alpha_j}\gamma h_{ij}^3, \nonumber \\
-\frac{\partial h_{ij}}{\partial \mathbf{x}_j}^T\dot{\mathbf{x}}_j\leq \frac{\alpha_j}{\alpha_i+\alpha_j}\gamma h_{ij}^3, \nonumber
\end{eqnarray}
  \item Strategy B distributes $b_{ij}$ to two robot agents. 
\begin{eqnarray} \label{eqn:cbfdecb}
- \Delta \mathbf{p}_{ij}^T\mathbf{u}_i \leq \frac{\alpha_i}{\alpha_i+\alpha_j}b_{ij}, \nonumber \\
\Delta \mathbf{p}_{ij}^T\mathbf{u}_j \leq \frac{\alpha_j}{\alpha_i+\alpha_j}b_{ij}, \nonumber
\end{eqnarray}
  \item Strategy C is a hybrid approach, which is inspired by strategies A and B. It partitions the terms related to acceleration limits of robot agents in (\ref{eqn:cbf3}) and distributes other terms appropriately like strategies A.
\begin{eqnarray}
&- \Delta \mathbf{p}_{ij}^T \mathbf{u}_i + \frac{\Delta \mathbf{p}_{ij}^T\Delta \mathbf{v}_{ij}}{\|\Delta \mathbf{p}_{ij}\|^2}\Delta \mathbf{p}_{ij}^T\mathbf{v}_i - \Delta \mathbf{v}_{ij}^T\mathbf{v}_i \nonumber\\
& \leq \frac{\alpha_i}{\alpha_i+\alpha_j} (\gamma h_{ij}^3\|\Delta \mathbf{p}_{ij}\| + \frac{\sqrt{\alpha_i+\alpha_j}\Delta \mathbf{p}_{ij}^T\Delta\mathbf{v}_{ij} }{\sqrt{2(\|\Delta \mathbf{p}_{ij}\|-D_s)}}), \label{eqn:hybridi} \\
& \Delta \mathbf{p}_{ij}^T \mathbf{u}_j - \frac{\Delta \mathbf{p}_{ij}^T\Delta \mathbf{v}_{ij}}{\|\Delta \mathbf{p}_{ij}\|^2}\Delta \mathbf{p}_{ij}^T\mathbf{v}_j + \Delta \mathbf{v}_{ij}^T\mathbf{v}_j \nonumber\\
& \leq  \frac{\alpha_j}{\alpha_i+\alpha_j} (\gamma h_{ij}^3\|\Delta \mathbf{p}_{ij}\|   + \frac{\sqrt{\alpha_i+\alpha_j}\Delta \mathbf{p}_{ij}^T\Delta\mathbf{v}_{ij} }{\sqrt{2(\|\Delta \mathbf{p}_{ij}\|-D_s)}}), \label{eqn:hybridj}
\end{eqnarray}
\end{enumerate}

These three decentralization strategies differ from each other in the required amount of information to implement the safety barrier certificates as shown in TABLE \ref{tb:compareparam}. The required information is categorized into self known parameters, sensing data and neighbors' parameters. The self known parameters and sensing data can be easily attained by the controller. Meanwhile, obtaining neighboring agents' parameters, e.g. acceleration limit $\alpha_j$, requires identity recognition or communication. Swarm robots are usually designed to be simple with limited hardware resources. In terms of required information, strategy C surpasses A and B by not requiring neighbors' parameters. Handling unknown neighboring agents' safety barrier parameters using strategy C is the topic of Section \ref{sec:estsfc}.
\begin{table}[h]
\caption{Comparison of required information for the implementation of three decentralization strategies}
\label{tb:compareparam}
\begin{center} 
\begin{tabular}{ | c | c| c| c | c |} 
\hline Strategy & Self params & Sensing data & Neighbors' params\\ 
\hline A &  $\alpha_i,\gamma$&$\Delta \mathbf{p}_{ij},\Delta \mathbf{v}_{ij},\mathbf{v}_i$ &$\alpha_j$ \\ 
\hline B & $\alpha_i,\gamma$ &$\Delta \mathbf{p}_{ij},\Delta \mathbf{v}_{ij}$ &$\alpha_j$  \\ 
\hline C & $\alpha_i,\gamma$ &$\Delta \mathbf{p}_{ij},\Delta \mathbf{v}_{ij}, \mathbf{v}_{i},$ &  \\ \hline 
\end{tabular} 
\end{center}
\end{table}

All of the three decentralization strategies still guarantees safety, if the controller follows safety barrier constraints. This is true because they partition the admissible control space to each agent, while safety barrier constraint (\ref{eqn:cbf3}) still holds. With strategy C, we can come up with a decentralized QP-based controller that is minimally invasive to pre-designed controller and provably safe.
\begin{equation}
\label{eqn:QPcontrollerDec}
 \begin{aligned}
\mathbf{u}_i^* =  & \:\: \underset{\mathbf u_i}{\text{argmin}}
 & & J(\mathbf u_i) =  \|{\mathbf u}_{i} - \hat{\mathbf u}_{i} \| &\\
 & \qquad \text{s.t.}
 & & \bar{A}_{ij}\mathbf{u}_i \leq \bar{b}_{ij},\: \qquad  \forall j\neq i,  &  \\
 &
 & &    \| \mathbf u_i\|_\infty  \leq \alpha_i,
 \end{aligned}
\end{equation}
where $\bar{A}_{ij} = -\Delta \mathbf{p}_{ij}^T$, $\bar{b}_{ij} = - \frac{\Delta \mathbf{p}_{ij}^T\Delta \mathbf{v}_{ij}}{\|\Delta \mathbf{p}_{ij}\|^2}\Delta \mathbf{p}_{ij}^T\mathbf{v}_i + \Delta \mathbf{v}_{ij}^T\mathbf{v}_i + \frac{\alpha_i}{\alpha_i+\alpha_j} (\gamma h_{ij}^3\|\Delta \mathbf{p}_{ij}\|   + \frac{\sqrt{\alpha_i+\alpha_j}\Delta \mathbf{p}_{ij}^T\Delta\mathbf{v}_{ij} }{\sqrt{2(\|\Delta \mathbf{p}_{ij}\|-D_s)}})$.

Fig. \ref{fig:sim2agents} illustrates a test case showing how the safety barrier certificates interact with pre-designed controller to guarantee safety. The pre-designed controller is a goal-to-goal controller without considering collision avoidance. The agent moved straight towards the goal when it is executed (Fig. \ref{fig:sim1}). When agents were about to collide, the barrier certificates automatically took over and computed an appropriate way to avoid collision while honoring the pre-designed control as much as possible(Fig. \ref{fig:sim2}). In the given case, agents successfully completed task and avoided collision. This is achieved by solving a simple online QP without computing the complicated forward reachable set.    
\begin{figure}[h]
\centering
\begin{subfigure}{.23\textwidth}
  \centering
  \resizebox{1.6in}{!}{\includegraphics{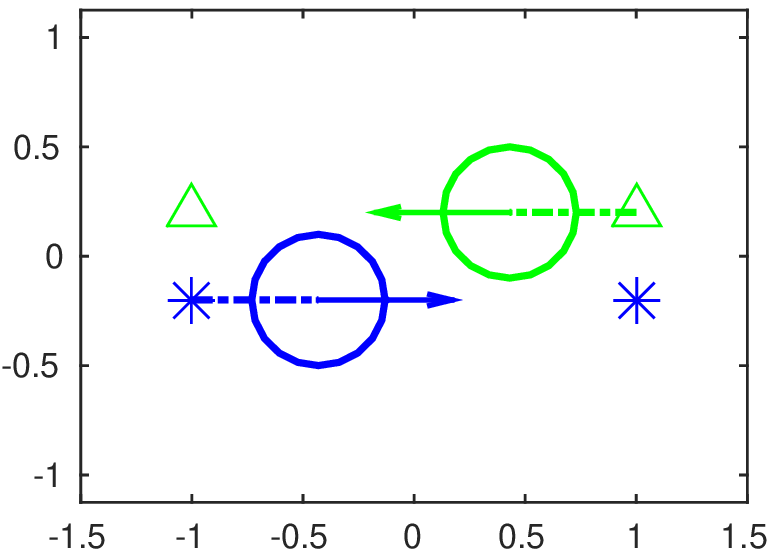}}
  \caption{Pre-designed control is used}
  \label{fig:sim1}
\end{subfigure}
\begin{subfigure}{.23\textwidth}
  \centering
  \resizebox{1.5in}{!}{\includegraphics{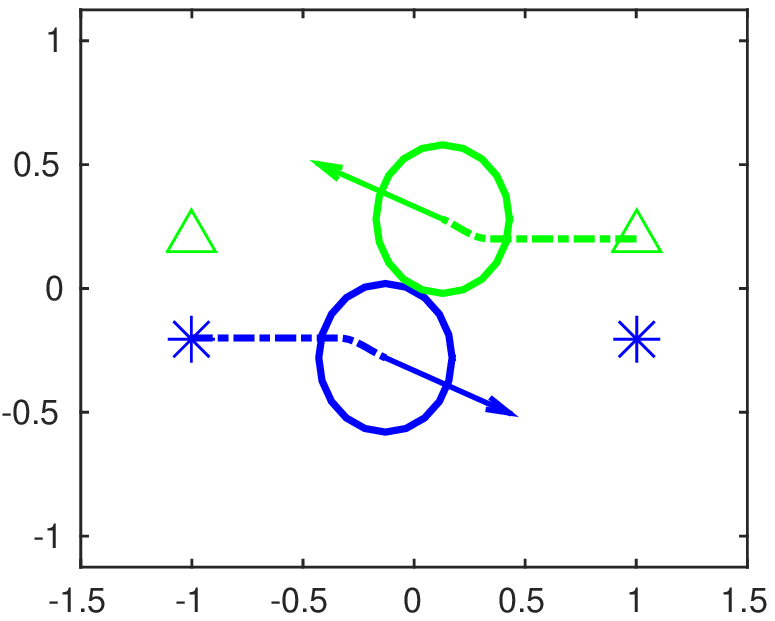}}
  \caption{Barrier certificates take over}
  \label{fig:sim2}
\end{subfigure}%
\begin{subfigure}{.23\textwidth}
  \centering
  \resizebox{1.6in}{!}{\includegraphics{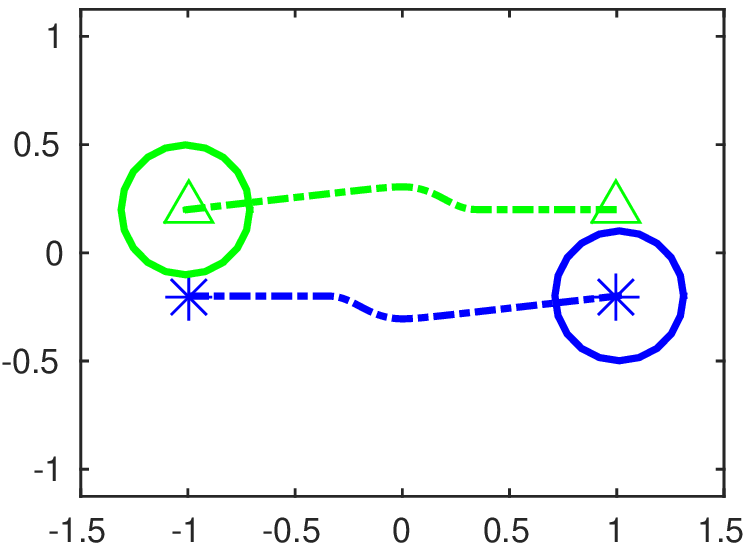}}
  \caption{Task Complete}
  \label{fig:sim3}
\end{subfigure}%
\caption{Two robot agents regulated by safety barrier certificates. The circles, arrows and dash-dot lines represent the agents' safety margin, current velocity and trajectories respectively.}
\label{fig:sim2agents}
\end{figure}

\section{Barrier Certificates with Unknown Parameters} \label{sec:estsfc}
Heterogeneity in agents' dynamical capabilities brings extra challenge to collision avoidance. In the robotic swarm, agents need to first assess how effective other agents can respond to safety threats before making decision for collision avoidance. Meanwhile, swarm robots are often designed to be simple and therefore lack the ability to obtain other agents' parameters. This section addresses the scenario that agents need to ensure safety when some dynamical parameters of other agents are unknown.
\subsection{Barrier Certificates with Different $\gamma$}
The safety barrier parameter $\gamma$ determines how fast the agents' states can approach the boundary of safe operation set $\mathcal{C}$. It turns out that agents with different $\gamma$ are still safe when running the decentralized barrier certificates.
\begin{lemma} \label{lemma:heterogamma}
Two heterogeneous agents $i,j\in\mathcal M$ regulated by safety barrier certificates (\ref{eqn:QPcontrollerDec}) with different parameters $\gamma_i, \gamma_j$ are guaranteed to be safe.
\end{lemma}
\begin{proof}
Agent $i$ and $j$ follow the safety barrier constriant given in (\ref{eqn:hybridi}) and (\ref{eqn:hybridj}) with different parameters $\gamma_i, \gamma_j$. Adding the two safety barrier constraints together gives:
\begin{eqnarray}\label{eqn:cbfgamma}
- \Delta \mathbf{p}_{ij}^T\Delta \mathbf{u}_{ij} \leq  \frac{\gamma'}{B_{ij}}h_{ij}^2\|\Delta \mathbf{p}_{ij}\| - \frac{(\Delta \mathbf{v}_{ij}^T\Delta \mathbf{p}_{ij})^2}{\|\Delta \mathbf{p}_{ij}\|^2} \nonumber \\+ \|\Delta \mathbf{v}_{ij}\|^2 + \frac{(\alpha_i+\alpha_j)\Delta \mathbf{v}_{ij}^T\Delta \mathbf{p}_{ij} }{\sqrt{2(\alpha_i+\alpha_j)(\|\Delta \mathbf{p}_{ij}\|-D_s)}},
\end{eqnarray}
where $\gamma' = \frac{\alpha_i\gamma_i + \alpha_j\gamma_j}{\alpha_i+\alpha_j}$. This inequality can be rewritten as $-\dot{h}_{ij}\leq \gamma'h_{ij}^3$, which guarantees safety as if a weighted version of $\gamma$ is used in the safety barrier certificates.
\end{proof}
This lemma provides the freedom for heterogeneous agents to choose or adaptively change their own $\gamma$ without worrying about endangering safety. The designer can use $\gamma$ to intentionally prioritize certain agents over others, which resembles the real life case of ambulance granted higher priority to go through the traffic flow.

Fig. (\ref{fig:kh}) demonstrates how heterogeneous $\gamma$ in safety barrier certificates can be used to coordinate conflicting agents. Two agents executing goal-to-goal controller regulated by heterogeneous barrier certificates are simulated in three different scenarios. When both agents adopt the same safety barrier parameter $\gamma$, they have neutral behavior when their goals conflict as shown in Fig. (\ref{fig:khm}). When the left agent uses larger $\gamma$, it moves straight to its goal aggressively, while the other agent moves around it to avoid collision (Fig. (\ref{fig:khl})). When the left agent is assigned with smaller $\gamma$, the left agent gives way to other agent when their goals conflict (Fig. (\ref{fig:khr})).
\begin{figure}[h]
\centering
\begin{subfigure}{.23\textwidth}
  \centering
  \resizebox{1.5in}{!}{\includegraphics{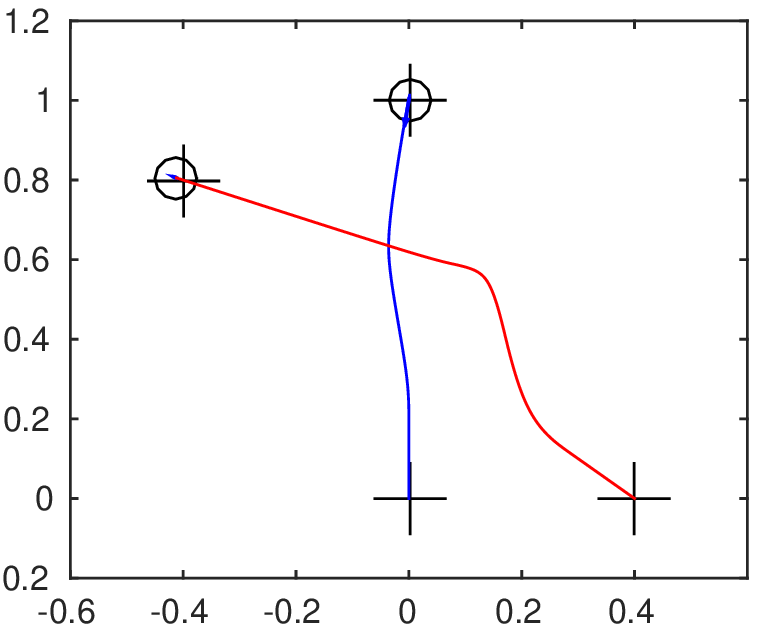}}
  \caption{Both agents are neutral}
  \label{fig:khm}
\end{subfigure}
\begin{subfigure}{.25\textwidth}
  \centering
  \resizebox{1.8in}{!}{\includegraphics{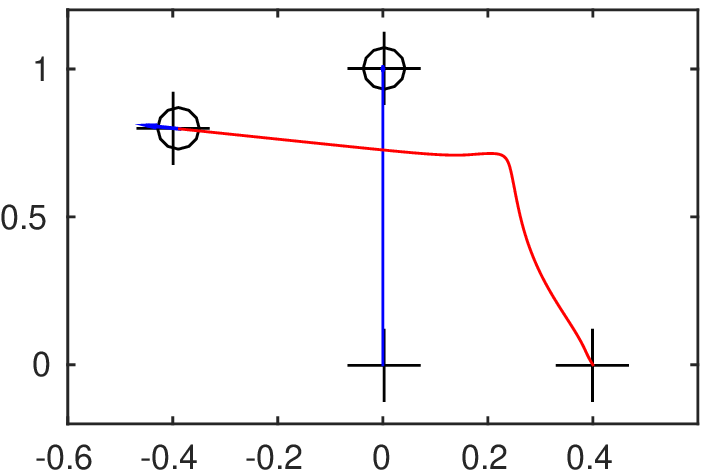}}
  \caption{Left agent aggressive}
  \label{fig:khl}
\end{subfigure}%
\begin{subfigure}{.23\textwidth}
  \centering
  \resizebox{1.5in}{!}{\includegraphics{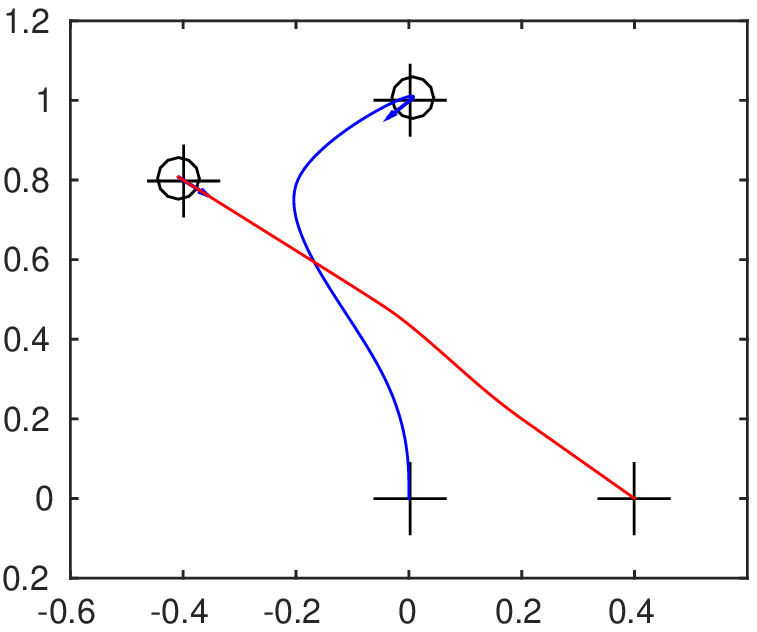}}
  \caption{Left agent conservative}
  \label{fig:khr}
\end{subfigure}%
\caption{Trajectories of two agents regulated by safety barrier certificates with different parameter $\gamma$}
\label{fig:kh}
\vspace{-0.2in}
\end{figure}

With heterogeneous safety barrier certificates, we can define the notion of neighbors to reduce the pairs of necessary safety barrier constraints:
\begin{align}\label{eqn:heteroneig}
\mathcal{N}_i &= \{j\in \mathcal{M}, j\neq i \: | \: \| \Delta \mathbf{p}_{ij}\|\leq D_\mathcal{N}^i, D_\mathcal{N}^i = D_s \nonumber \\ &+ \frac{1}{2(\alpha_i+\alpha_{\min})}(\sqrt[3]{\frac{2(\alpha_i+\alpha_{\min})}{\gamma_i}} + \beta_i + \beta_{\max})^2 \},
\end{align}
where $\alpha_{\min} = \underset{j\in \mathcal{M}, j\neq i}{\min}\{\alpha_j\}$ and $\beta_{\max} = \underset{j\in \mathcal{M}, j\neq i}{\max}\{\beta_j\}$ are the minimal acceleration limit and maximum speed limit of other agents in the robotic swarm. When $\alpha_{\min}$ and $\beta_{\max}$ are unknown, the most conservative values can be used, i.e. lower bound of $\alpha_{\min}$ and upper bound of $\beta_{\max}$. The neighbor's notion is helpful in reducing computation intensity and sensing requirement. This notion is valid because there is no threat of collision when agents are sufficiently far away from each other.
\begin{theorem}
Any agent $i\in \mathcal M$ is guaranteed to be safe if it only forms ZCBFs with its heterogeneous neighbors defined by (\ref{eqn:heteroneig}).
\end{theorem}
\begin{proof}
Heterogeneous agents each possesses a safety neighbor disk with different radius. Thus there are generally three scenarios considering $\forall j\in\mathcal M, j\neq i$, i.e. $\|\Delta \mathbf p_{ij}\|>\max\{D_\mathcal{N}^i,D_\mathcal{N}^j\}$, $\max\{D_\mathcal{N}^i,D_\mathcal{N}^j\}\geq\|\Delta \mathbf p_{ij}\|\geq \min\{D_\mathcal{N}^i,D_\mathcal{N}^j\}$ or $\|\Delta \mathbf p_{ij}\|<\min\{D_\mathcal{N}^i,D_\mathcal{N}^j\}$.
\begin{figure}[h]
  \centering
  \psfrag{1}{\small{$D_\mathcal{N}^j$}}
  \psfrag{2}{\small{$D_\mathcal{N}^i$}}
  \psfrag{3}{\small{$\Delta \mathbf p_{ij}$}}
  \resizebox{1.8in}{!}{\includegraphics{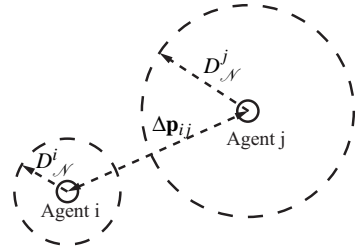}}
  \label{fig:CBFHetero1}
  \caption{Two heterogeneous agents with different safety neighbor disks}
\end{figure}

When $\|\Delta \mathbf p_{ij}\|>\max\{D_\mathcal{N}^i,D_\mathcal{N}^j\}$, both agents are not within neighbor's range. We can prove $-\dot{h}_{ij} \leq \min\{\gamma_i,\gamma_j\}h_{ij}^3$ following similar reasoning of \emph{Theorem} \cite{borrmann2015Swarm} by considering the worst-case scenario.

When $\max\{D_\mathcal{N}^i,D_\mathcal{N}^j\}\geq\|\Delta \mathbf p_{ij}\|\geq \min\{D_\mathcal{N}^i,D_\mathcal{N}^j\}$, one agent is a neighbor of the other, while the other agent is not. Following \emph{Theorem} 1, we can prove $-\dot{h}_{ij} < \gamma_k {h_{ij}^3}$, where $k=\underset{k\in\{i,j\}}{\text{argmin}} \,D_\mathcal{N}^k$. When $\|\Delta \mathbf p_{ij}\|<\min\{D_\mathcal{N}^i,D_\mathcal{N}^j\}$, both agents are neighbors of each other. Following \emph{Lemma} \ref{lemma:heterogamma}, it is guaranteed that $-\dot{h}_{ij}\leq \gamma' h_{ij}^3$.

To sum up, safety is guaranteed in all three cases with different ZCBF parameters. Heterogeneous agents only needs to form ZCBFs with their neighbors to guarantee safety.
\end{proof}

\subsection{Barrier Certificates with Unknown Acceleration Limits}
As discussed in Section \ref{sec:DecHeteroBar}, identifying the acceleration limits of neighboring agents can be a complicated problem. When no prior knowledge about other agents' acceleration limits is provided, we will prove that estimated values can be used. Consequently, the safe set definition will be slightly different for different agents. Let $\alpha_i$ and $\alpha_{ij}$ be agent $i$'s acceleration limit and estimate of agent $j$'s acceleration limit. The pairwise safe operation set $\bar{\mathcal C}_{ij}$ of agent $i$ is:
\vspace{-0.1in}
\begin{align*}
\bar{\mathcal C}_{ij} &= \{(\mathbf p_i, \mathbf v_i)~|~h_{ij}(\alpha_i+ \alpha_{ij}) = \frac{\Delta \mathbf{p}_{ij}^T}{\|\Delta \mathbf{p}_{ij}\|}\Delta \mathbf{v}_{ij} \\ &+ \sqrt{2(\alpha_i + \alpha_{ij})(\|\Delta \mathbf{p}_{ij}\|-D_s)}\geq 0 \}, j\neq i,
\end{align*}
With the estimated parameters, the hybrid strategy in Section \ref{sec:DecHeteroBar} for distributing safety barrier certificates is:
\vspace{-0.1in}
\begin{flalign}
&- \Delta \mathbf{p}_{ij}^T \mathbf{u}_i + \frac{\Delta \mathbf{p}_{ij}^T\Delta \mathbf{v}_{ij}}{\|\Delta \mathbf{p}_{ij}\|^2}\Delta \mathbf{p}_{ij}^T\mathbf{v}_i - \Delta \mathbf{v}_{ij}^T\mathbf{v}_i \leq \nonumber\\
 &\frac{\alpha_i}{\alpha_i+\alpha_{ij}} (\gamma_i h_{ij}^3(\alpha_i+ \alpha_{ij})\|\Delta \mathbf{p}_{ij}\| + \frac{\sqrt{\alpha_i+\alpha_{ij}}\Delta \mathbf{p}_{ij}^T\Delta \mathbf{v}_{ij} }{\sqrt{2(\|\Delta \mathbf{p}_{ij}\|-D_s)}}) \label{eqn:cbfdecc1} \\
& \Delta \mathbf{p}_{ij}^T \mathbf{u}_j - \frac{\Delta \mathbf{p}_{ij}^T\Delta \mathbf{v}_{ij}}{\|\Delta \mathbf{p}_{ij}\|^2}\Delta \mathbf{p}_{ij}^T\mathbf{v}_j + \Delta \mathbf{v}_{ij}^T\mathbf{v}_j \leq \nonumber\\
&\frac{\alpha_j}{\alpha_j+\alpha_{ji}} (\gamma_j h_{ji}^3(\alpha_j+ \alpha_{ji})\|\Delta \mathbf{p}_{ij}\|   + \frac{\sqrt{\alpha_j+\alpha_{ji}}\Delta \mathbf{p}_{ij}^T\Delta \mathbf{v}_{ij} }{\sqrt{2(\|\Delta \mathbf{p}_{ij}\|-D_s)}}) \hspace{-0.1in} \label{eqn:cbfdecc2}
\end{flalign}

Before introducing the estimation method, we need to make sure that safety is still guaranteed when imperfect estimation parameters are used. In order to guarantee safety with inaccurate parameters, it is desirable to ensure that $\bar{\mathcal C}_{ij}$ is always subset of $\mathcal C_{ij}$. Notice that when $\alpha_{ji} \leq \alpha_i, \alpha_{ij} \leq \alpha_j$, we have $\bar{\mathcal C}_{ij} \subseteq \mathcal C_{ij}$. It is intuitive to guess that agents are safe if conservative estimates of neigboring agents' acceleration limits are used. We will prove this intuition for decentralized heterogeneous barrier certificates.
\begin{lemma} \label{lemma:hybridest}
If $\alpha_{ji} \leq \alpha_i, \alpha_{ij} \leq \alpha_j$ and the safety barrier constraints in (\ref{eqn:cbfdecc1}),(\ref{eqn:cbfdecc2}) are satisfied, safety is still guaranteed.
\end{lemma}

\begin{proof}
When agents $i$ and $j$ use their own estimates of acceleration limits based on (\ref{eqn:cbfdecc1}) and (\ref{eqn:cbfdecc2}), we can get:
\begin{flalign}\label{eqn:cbfestimate}
&- \Delta \mathbf{p}_{ij}^T \Delta\mathbf{u}_{ij} + \frac{(\Delta \mathbf{p}_{ij}^T\Delta \mathbf{v}_{ij})^2}{\|\Delta \mathbf{p}_{ij}\|^2} - \Delta \mathbf{v}_{ij}^T\Delta\mathbf{v}_{ij} \nonumber\\
&\leq \frac{\alpha_i\gamma_i h_{ij}^3(\alpha_i+\alpha_{ij})}{\alpha_i+\alpha_{ij}} \|\Delta \mathbf{p}_{ij}\| + \frac{\alpha_j\gamma_j h_{ji}^3(\alpha_j+\alpha_{ji})}{\alpha_j+\alpha_{ji}} \|\Delta \mathbf{p}_{ij}\| \nonumber\\
&+ (\frac{\alpha_i}{\sqrt{\alpha_i+\alpha_{ij}}}+\frac{\alpha_j}{\sqrt{\alpha_j+\alpha_{ji}}})\frac{\Delta \mathbf{p}_{ij}^T\Delta \mathbf{v}_{ij} }{\sqrt{2(\|\Delta \mathbf{p}_{ij}\|-D_s)}},  
\end{flalign}
Recall that if perfect parameter estimation is achieved ($\alpha_{ji} = \alpha_i, \alpha_{ij} = \alpha_j$), (\ref{eqn:cbfestimate}) is identical to (\ref{eqn:cbfgamma}). Next, we will discuss safety under two different scenarios where two agents are moving closer or further away from each other:

\begin{enumerate}
\item when $\Delta \mathbf{p}_{ij}^T\Delta \mathbf{v}_{ij}\leq 0$, agents $i$  and $j$ are moving closer to each other. With $\alpha_{ji} \leq \alpha_i, \alpha_{ij} \leq \alpha_j$, we have $ \frac{\alpha_i}{\sqrt{\alpha_i+\alpha_{ij}}} +\frac{\alpha_j}{\sqrt{\alpha_j+\alpha_{ji}}} \geq \sqrt{\alpha_i + \alpha_j}$. Thus
\vspace{-0.1in}
\begin{eqnarray}\label{eqn:cbfestimate3}
&- \Delta \mathbf{p}_{ij}^T \Delta\mathbf{u}_{ij} + \frac{(\Delta \mathbf{p}_{ij}^T\Delta \mathbf{v}_{ij})^2}{\|\Delta \mathbf{p}_{ij}\|^2} - \Delta \mathbf{v}_{ij}^T\Delta\mathbf{v}_{ij} \nonumber\\
 \leq & \bar{\gamma} h_{ij}^3(\alpha_i+\alpha_j)\|\Delta \mathbf{p}_{ij}\| + \frac{\sqrt{\alpha_i+\alpha_j}\Delta \mathbf{p}_{ij}^T\Delta \mathbf{v}_{ij} }{\sqrt{2(\|\Delta \mathbf{p}_{ij}\|-D_s)}},  
\end{eqnarray}
where $\bar{\gamma}=\frac{\alpha_i\gamma_i }{\alpha_i+\alpha_{ij}} + \frac{\alpha_j\gamma_j }{\alpha_j+\alpha_{ji}}$. Compared with (\ref{eqn:cbf3}), this inequality can be rewritten as $-\dot{h}_{ij}(\alpha_i +\alpha_j)\leq \bar{\gamma}{h_{ij}(\alpha_i + \alpha_j)^3}$, which guarantees safety as if a weighted version of $\gamma$ is adopted. This means that, if $\Delta \mathbf{p}_{ij}^T\Delta \mathbf{v}_{ij}\leq 0$, the forward invariance of the nominal safe operation set $\mathcal{C}$ is guaranteed. 

\item when $\Delta \mathbf{p}_{ij}^T\Delta \mathbf{v}_{ij}> 0$ (agents are moving away from each other), it is guaranteed to have $h_{ij}(\alpha_i +\alpha_j) \geq 0$. Thus agents always stays in the nominal safe operation set $\mathcal{C}$ in this scenario.
\end{enumerate}

Now we are one step away from proving the safety guarantee for all cases. Because agents are switching back and forth between the two cases. The switchings might compromise safety. In case (1), forward set invariance requires agent $i$ to always start in $\mathcal{C}$ after each switching. Due to the second order dynamical model used for barrier certificates, $\Delta \mathbf{p}_{ij}^T\Delta \mathbf{v}_{ij}$ is continuous with respect to time. Thus the switching between two cases always occurs at $\Delta \mathbf{p}_{ij}^T\Delta \mathbf{v}_{ij}= 0$, where $h_{ij}(\alpha_i +\alpha_j) \geq 0$. Combining the two scenarios with the safe switching condition, agent $i$ is guaranteed to be safe with respect to the nominal safe operation set $\mathcal{C}$.
\end{proof}

Now, what is left to do is to find an appropriate estimation method to estimate the acceleration limits of neighboring agents with local observation. With the local sensor measurements of neighboring agents, we can construct a distributed least squares estimator or Kalman filter \cite{mesbahi2010graph} to estimate the current acceleration $\|\bar{\mathbf u}_j\|$ of agent $j$. The steps to update the estimate of $\alpha_j$ can be designed as:
\begin{enumerate}
\item Set conservative initial guess as $\alpha_{ij}(t_0)=\alpha_{min}$.
\item Use local observations to update $\|\bar{\mathbf u}_j\|$.
\item Update $\alpha_{ij}$ with $\dot{\alpha}_{ij} = k(\max\{\alpha_{ij},\|\bar{\mathbf u}_j\|\} - \alpha_{ij})$, where $k$ depends on the accuracy of local observations. 
\end{enumerate}
note that this strategy will ensure that parameter estimation satisfies $\alpha_{ji} \leq \alpha_i, \alpha_{ij} \leq \alpha_j$. Thus safety is still guaranteed using the estimated parameters due to \textit{lemma} \ref{lemma:hybridest}. With this estimation strategy, agents do not need to know the acceleration limits of neighboring agents. They can start with conservative initial guesses, and gradually improve their knowledge with local observations without endangering safety.

\section{Simulation results} \label{sec:sim}
A multi-robot system with six heterogeneous agents is simulated with MATLAB. Each agent is modelled with double integrator dynamics and executes a goal-to-goal controller without considering collision avoidance. This system contains two types of agents: small agile agents ($\alpha_{s} = 1.2~m/s^2$, safety radius is $0.2~m$); large cumbersome agent ($\alpha_{l} = 0.6~m/s^2$, safety radius is $0.4~m$). As illustrate in Fig.\ref{fig:sim6agents}, the objective of the pre-designed controller is to make all agents exchange position with the agents on the opposite side of the large circle. Without collision avoidance strategy, the goal-to-goal controller will lead to collision of all agents in the middle.

The heterogeneous safety barrier certificates are implemented side by side with the pre-designed control strategy. All agents started heading towards the center following the goal-to-goal controller (Fig. \ref{fig:t1}). As they moved closer to each other, the safety barrier certificates were activated and kept all agents with enough safety distance away from each other (Fig. \ref{fig:t2}). The small agents are more agile and took up more responsibilities in avoiding collision, while the cumbersome agent decelerated but still continued its own path because of its large inertia (Fig. \ref{fig:t3}). When the large agent was about to reach its destination, its speed is almost zero. Other small agents were safe to pursue their own goals without worrying about colliding with the large agent (Fig. \ref{fig:t4}). At last, the heterogeneous safety barrier certificates successfully helped navigate all agents out of the ``crowded" scenarios and achieved their individual objectives.

Note that the core of safety barrier certificates is a QP-based controller, which can be executed very efficiently. Compared with conventional multi-agent collision avoidance methods, the proposed method is more suitable for real-time application on large-scale robotic swarm because it does not require computation of complicated forward reachable set.
\begin{figure}[h]
\centering
\begin{subfigure}{.23\textwidth}
  \centering
  \resizebox{1.5in}{!}{\includegraphics{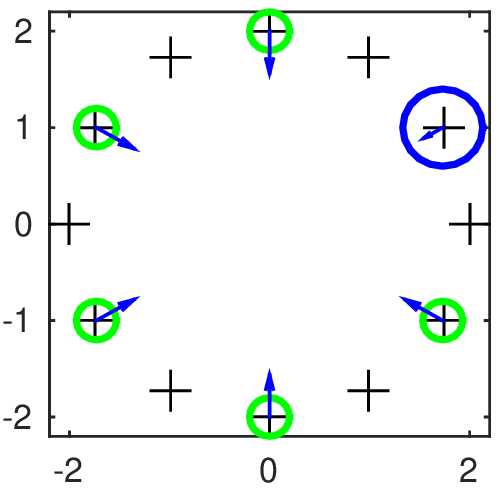}}
  \caption{Time step t = 10}
  \label{fig:t1}
\end{subfigure}
\begin{subfigure}{.23\textwidth}
  \centering
  \resizebox{1.5in}{!}{\includegraphics{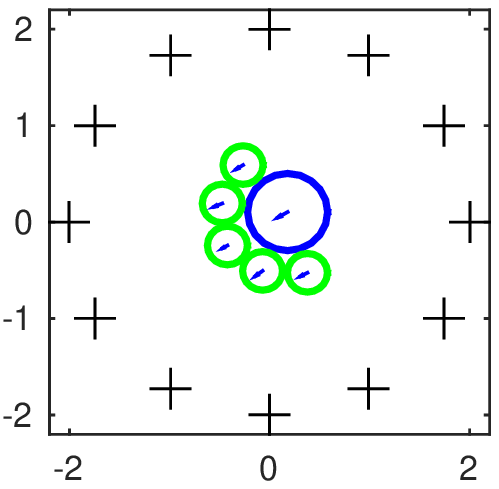}}
  \caption{Time step t = 520}
  \label{fig:t2}
\end{subfigure}%
\begin{subfigure}{.23\textwidth}
  \centering
  \resizebox{1.5in}{!}{\includegraphics{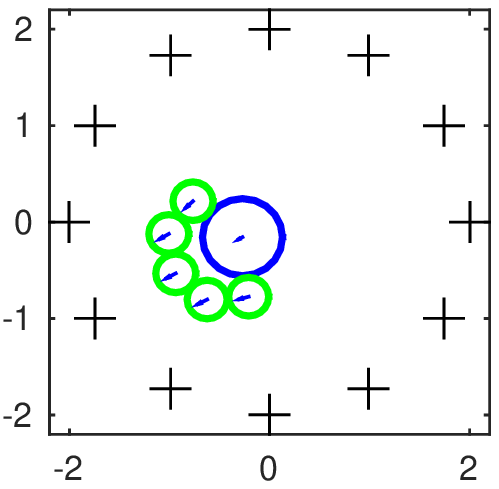}}
  \caption{Time step t = 1000}
  \label{fig:t3}
\end{subfigure}\\
\begin{subfigure}{.23\textwidth}
  \centering
  \resizebox{1.6in}{!}{\includegraphics{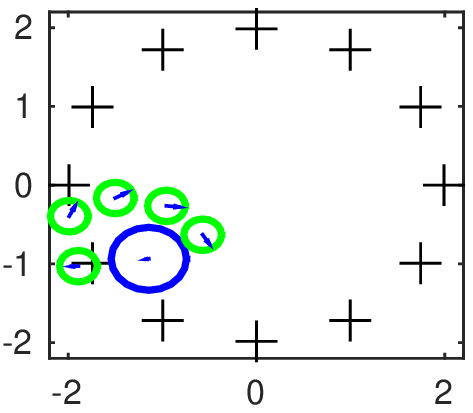}}
  \caption{Time step t = 2400}
  \label{fig:t4}
\end{subfigure}%
\begin{subfigure}{.23\textwidth}
  \centering
  \resizebox{1.5in}{!}{\includegraphics{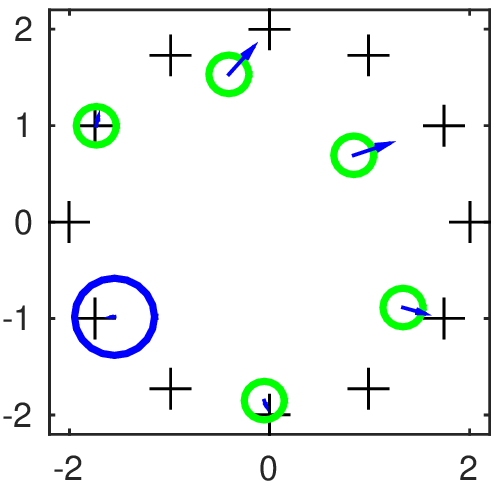}}
  \caption{Time step t = 2830}
  \label{fig:t5}
\end{subfigure}%
\caption{Six robot agents regulated by heterogeneous safety barrier certificates. The acceleration limits of small and large agents are $1.2~m/s^2$ and $0.6~m/s^2$. The speed limits of all agents are $0.6~m/s$. The small and large circles represent the safety radius of different agents, which are $0.2~m$ and $0.4~m$ respectively. Units for X and Y coordinates in the plots are in meters.}
\label{fig:sim6agents}
\end{figure}

\section{Experimental results} \label{sec:exp}
The heterogeneous safety barrier certificates were implemented on a heterogeneous robotic swarm with three Khepera III robots ($\alpha_K = 2.0 ~m/s^2$) and one Magellan Pro robot ($\alpha_M = 0.5 ~m/s^2$). The positions of robots are tracked by Optitrack motion capture system. Those two types of robots have distinct dimensions and dynamical capabilities. The diameters of Khepera III and Magellan Pro robots are $13~cm$ and $41~cm$. The actual dynamical model of mobile robots used in this experiment is unicycle model, which is approximated with double integrated dynamics using Lyapunov based approach. The pre-designed controller is a goal-to-goal controller ($\hat{\mathbf u}_i=-k_1(\mathbf p_i- \mathbf r_i) - k_2\mathbf v_i$), which exchanges the positions of agents on the diagonal line of a rectangle, without considering collision avoidance. The heterogeneous safety barrier certificates were executed as a lower level safety program with no knowledge about overall goal of the higher level controller.

Fig. \ref{fig:expKM} shows a overhead view of the robots during the experiment and plots of corresponding experimental data. All four robots started heading straightly towards the opposite side of the rectangle (Fig. \ref{fig:t1e}). The safety barrier was inactive because the pre-designed coordination control command is considered safe. When robots moved closer, the safety barrier interfered because collision was about to happen. As illustrated in (Fig. \ref{fig:t2e}), three Khepera III robots turned around to avoid collision, while the Magellan robot kept pushing forward. This is because Magellan Pro robot has more momentum and can not brake fast enough to avoid collision. Those more agile Khepera III robots carried more responsibilities in collision avoidance when Magellan Pro robot reacted slowly. When the Magellan Pro robot almost reached its goal position and became slower in motion, other Khepera III robots got the chance to pursue their goals (Fig. \ref{fig:t3e}). It can be observed that the safety barrier directed robots away from collision and computes the command that is closest to pre-designed control command. After robots navigated away from the ``crowded" area, the pre-designed controller took over again (Fig. \ref{fig:t4e}). Ultimately, all robots reached desired configuration, i.e. exchange position with robots on the opposite side (Fig. \ref{fig:t5e}). A video can be found online \cite{Hetero:video}.

\begin{figure}[H]
\centering
\begin{subfigure}{.20\textwidth}
  \centering
  \resizebox{1.6in}{!}{\includegraphics{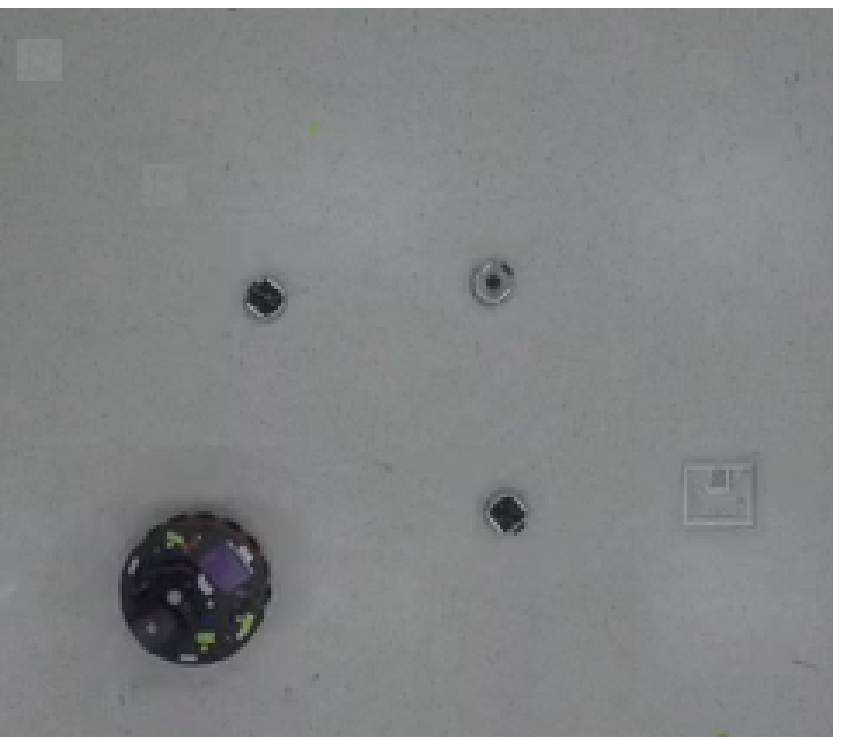}}
  \caption{Agents at 4.0s}
  \label{fig:t1e}
\end{subfigure}%
\hspace{0.02in}
\begin{subfigure}{.22\textwidth}
  \centering
  \resizebox{1.8in}{!}{\includegraphics{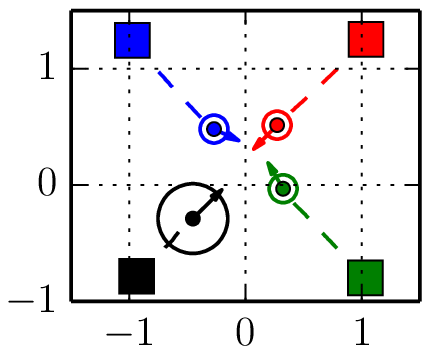}}
  \label{fig:t1ep}
\end{subfigure}%
\hspace{0.05in}
\begin{subfigure}{.20\textwidth}
  \centering
  \resizebox{1.6in}{!}{\includegraphics{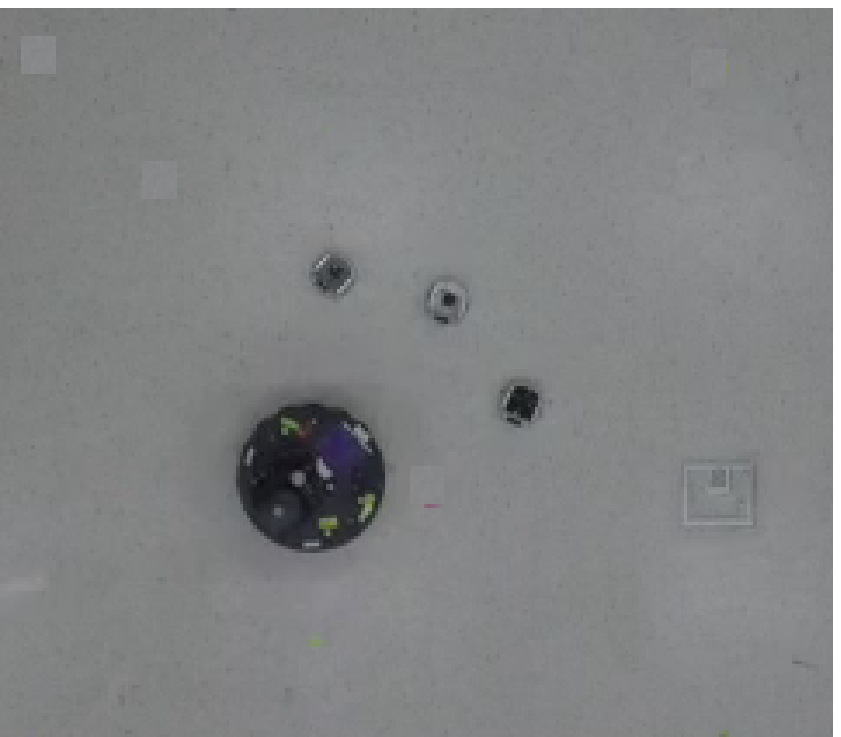}}
  \caption{Agents at 7.0s}
  \label{fig:t2e}
\end{subfigure}
\begin{subfigure}{.22\textwidth}
  \centering
  \resizebox{1.8in}{!}{\includegraphics{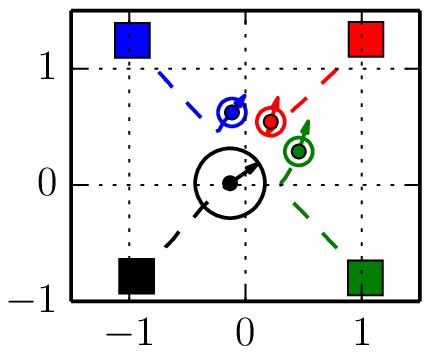}}
  \label{fig:t2ep}
\end{subfigure}
\hspace{0.05in}
\begin{subfigure}{.20\textwidth}
  \centering
  \resizebox{1.6in}{!}{\includegraphics{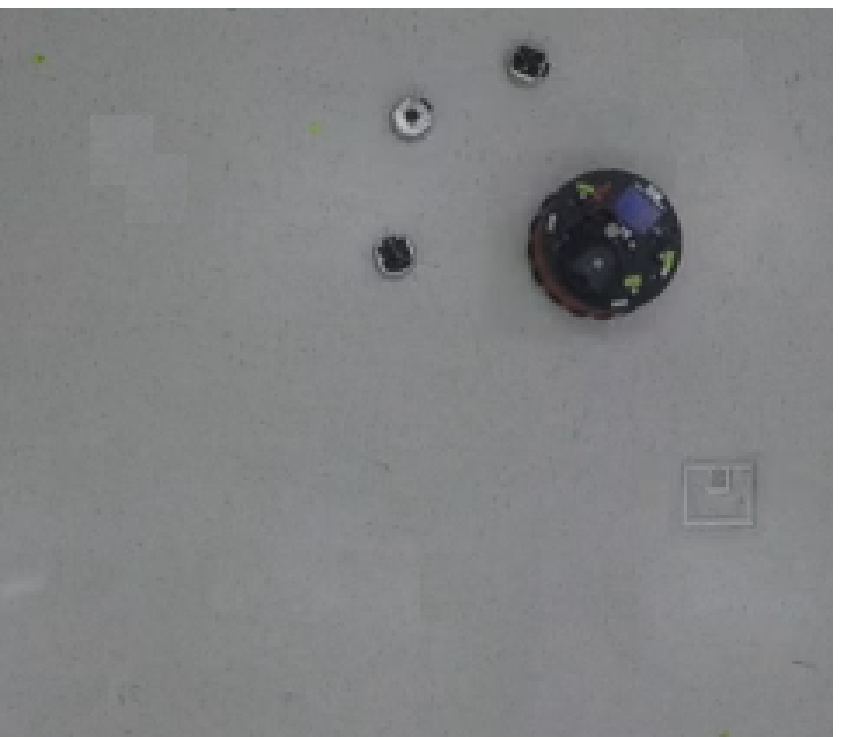}}
  \caption{Agents at 13.0s}
  \label{fig:t3e}
\end{subfigure}%
\hspace{0.02in}
\begin{subfigure}{.22\textwidth}
  \centering
  \resizebox{1.8in}{!}{\includegraphics{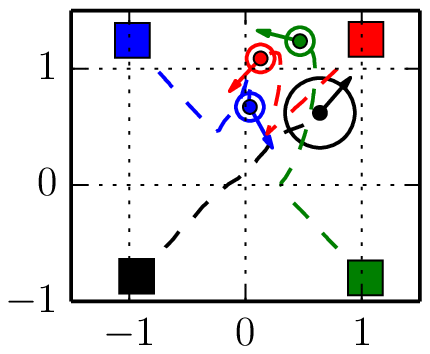}}
  \label{fig:t3ep}
\end{subfigure}%
\hspace{0.02in}
\begin{subfigure}{.20\textwidth}
  \centering
  \resizebox{1.6in}{!}{\includegraphics{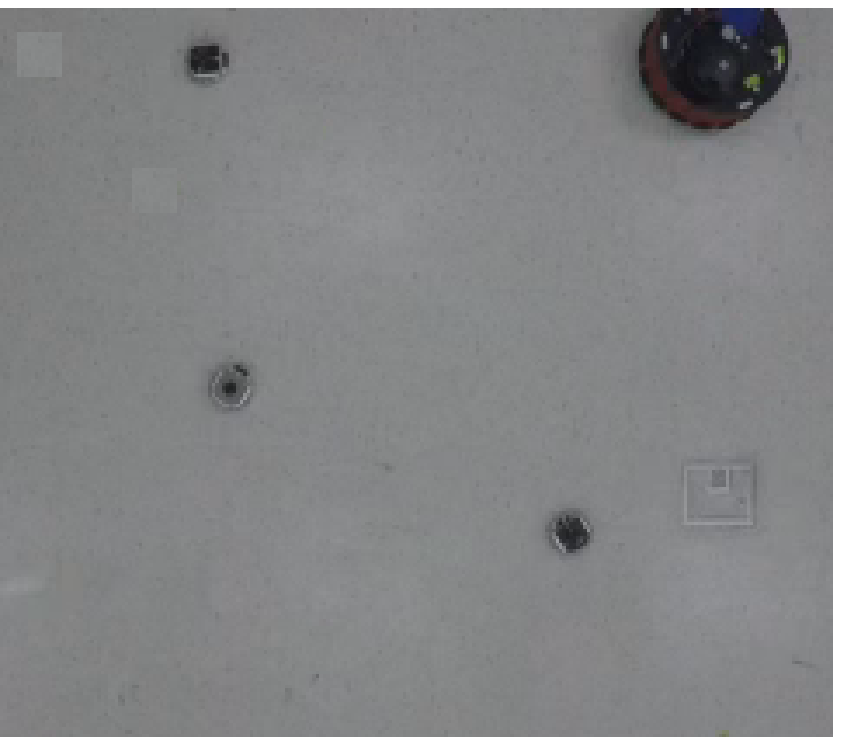}}
  \caption{Agent at 16.0s}
  \label{fig:t4e}
\end{subfigure}%
\hspace{0.02in}
\begin{subfigure}{.22\textwidth}
  \centering
  \resizebox{1.8in}{!}{\includegraphics{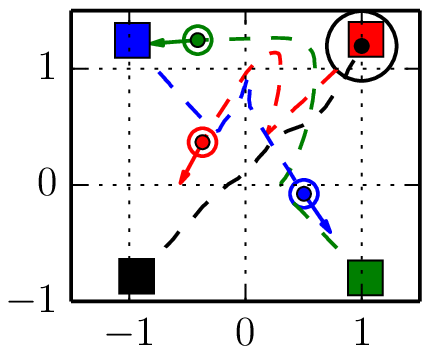}}
  \label{fig:t4ep}
\end{subfigure}
\begin{subfigure}{.20\textwidth}
  \centering
  \resizebox{1.6in}{!}{\includegraphics{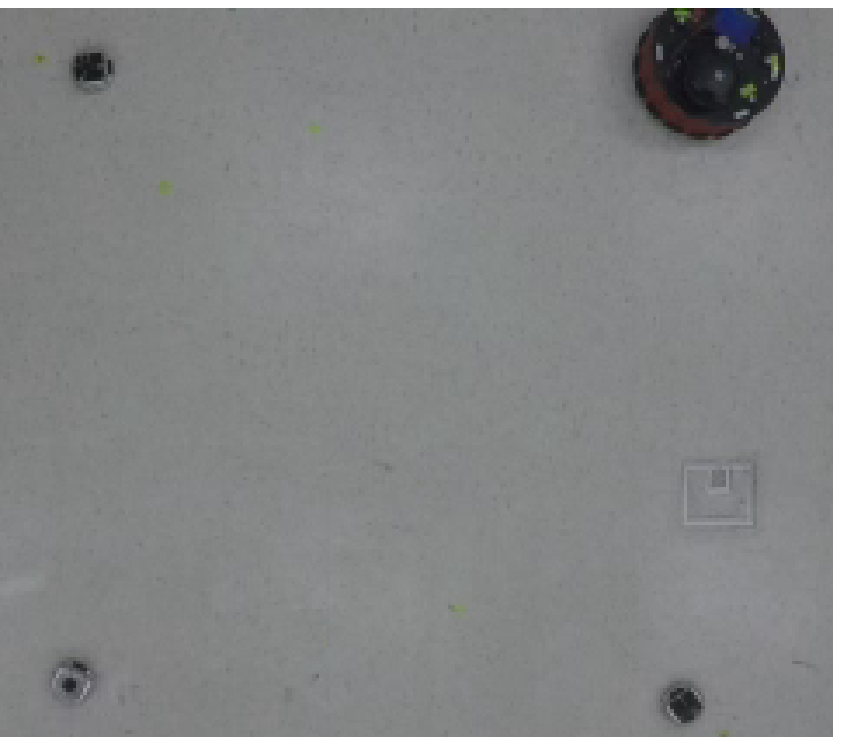}}
  \caption{Agent at 21.3s}
  \label{fig:t5e}
\end{subfigure}%
\vspace{-0.1in}
\begin{subfigure}{.22\textwidth}
  \centering
  \resizebox{1.8in}{!}{\includegraphics{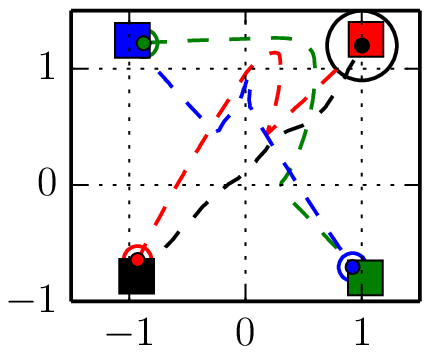}}
  \label{fig:t5ep}
\end{subfigure}%
\caption{Test run of three Khepera robots (small circles) and one Magellan robot (large circle) with heterogeneous safety barrier certificates. The arrow, circle and dashed line represent current velocity, position and trajectory of robot agents. The square markers stand for initial and goal positions.}
\label{fig:expKM}
\end{figure}


\section{Conclusion and future work} \label{sec:conclude}
The heterogeneous safety barrier certificates proposed in this paper provides a provable way to address the challenges in collision avoidance brought by heterogeneity in robots' dynamical capabilities. Both simulation and experimental results validate the effectiveness of the proposed approach. While studying those results, several interesting future research directions also arise. When the objectives of several agents conflict with each other, the agents sometimes get into a deadlock. When deadlock happens, safety is guaranteed but desired tasks can not be completed. It is important to design a strategy that breaks deadlock to ensure task completion. 

In some ``crowded" situations, several safety barrier constraints might conflict with each other, rendering the optimization-based controller infeasible. To remedy this problem, zeroing control barrier function is designed to pull the states of agents back to the safe operation set when violation occurs. However, for some safety critical systems, synthesizing safety barrier certificates that are guaranteed feasible is very significant.

\addtolength{\textheight}{-12cm}   


\bibliographystyle{abbrv}
\bibliography{mybib}
\end{document}